\documentclass{article}
\usepackage[preprint,nonatbib]{neurips_2020}
\usepackage{amsmath}
\usepackage{subcaption}
\usepackage{mathtools}
\usepackage{graphicx}
\usepackage[title]{appendix}
\usepackage{parskip}
\usepackage{hyperref}
\usepackage{tikz}

\usepackage{url} 
\usepackage{booktabs}       
\usepackage{nicefrac}       
\usepackage{subcaption}
\usepackage{amsthm}
\usepackage{amssymb}
\usepackage{amsfonts}
\usepackage{fontawesome}

\newtheorem{problem*}{Problem}
\newtheorem{theorem}{Theorem}
\newtheorem{lemma}{Lemma}

\newtheorem{definition}{Definition}

\newtheorem{corollary}{Corollary}

\newtheorem*{example*}{Example}
\usepackage[ruled,linesnumbered,noresetcount,vlined]{algorithm2e}
\makeatletter

\newcommand{\nosemic}{\renewcommand{\@endalgocfline}{\relax}}
\newcommand{\dosemic}{\renewcommand{\@endalgocfline}{\algocf@endline}}
\let\oldnl\nl
\newcommand{\nonl}{\renewcommand{\nl}{\let\nl\oldnl}}

\SetKwFor{Repeat}{repeat}{:}{endw}
\makeatother
\usepackage{xcolor}
\usepackage{setspace}
\usepackage{multirow}
\usepackage{bm,bbm}
\usepackage{paralist}
\usepackage{enumitem}

 \newcommand{\cD}{\mathcal{D}}

\newcommand{\cM}{\mathcal{M}} 
 
\newcommand{\cR}{\mathcal{R}}

\newcommand{\cK}{\mathcal{K}}

\newcommand{\norm}[1]{\left\lVert#1\right\rVert}
\newcommand{\bias}[1]{\operatorname{Bias}\left[#1\right]}
\newcommand{\pr}[1]{\operatorname{Pr}\left(#1\right)}
\newcommand{\var}[1]{\operatorname{Var}\left(#1\right)}
\newcommand{\EE}[2]{\operatorname{\mathbb{E}}_{#1}\left[#2\right]}
\newcommand{\inner}[2]{\left\langle#1,#2\right\rangle}
\newcommand{\err}[1]{\operatorname{Err}\left(#1\right)}

\newcommand{\BB}[2]{B_{#1}\left(#2\right)}

\newcommand{\truedata}{\bm{x}}
\newcommand{\noisydata}{\tilde{\bm{x}}}
\newcommand{\postdata}{\hat{\bm{x}}}
\newcommand{\nnpostdata}{\hat{\bm{x}}_{+}}
\newcommand{\sumpostdata}{\hat{\bm{x}}_S}
\newcommand{\sumpostdatai}{\hat{\bm{x}}_{Si}}
\newcommand{\region}{\cK}
\newcommand{\nnregion}{\cK}
\newcommand{\sumregion}{\cK}

\newcommand{\nnprogram}{P_{+}}

\newcommand{\radius}{r_m}
\newcommand{\bound}{C'}

\newcommand{\unit}{\bm{e}}
\newcommand{\newunit}{\bm{e}'}

\newcommand{\refopt}{\operatorname{Ref}}

\newcommand{\noise}{\bm{\eta}}

\newcommand{\RR}{\mathbb{R}}
\newcommand{\NN}{\mathbb{N}}

\makeatletter
\newcommand{\oset}[3][0ex]{%
  \mathrel{\mathop{#3}\limits^{
    \vbox to#1{\kern-2\ex@
    \hbox{$\scriptstyle#2$}\vss}}}}
\makeatother

\title{Bias and Variance of Post-processing in Differential Privacy}

\author{%
  Keyu Zhu\\
  Georgia Institute of Technology \\
  \texttt{kzhu67@gatech.edu}
  \And
  Pascal Van Hentenryck\\
  Georgia Institute of Technology\\
  \texttt{pvh@isye.gatech.edu}
  \And 
  Ferdinando Fioretto \\
  Syracuse University \\
  \texttt{ffiorett@syr.edu} 
}

\begin{document}
\maketitle\sloppy\allowdisplaybreaks

\begin{abstract}
Post-processing immunity is a fundamental property of differential
privacy: it enables the application of arbitrary data-independent 
transformations to the results of differentially private outputs 
without affecting their privacy guarantees. 
When query outputs must satisfy domain constraints, post-processing 
can be used to project the privacy-preserving outputs onto the feasible region. 
Moreover, when the feasible region is convex, a widely adopted class of post-processing steps is also guaranteed to improve accuracy. Post-processing has 
been applied successfully in many applications including census 
data-release, energy systems, and mobility. However, its effects on the 
noise distribution is poorly understood: It is often argued that 
post-processing may introduce bias and increase variance. This paper 
takes a first step towards understanding the properties of 
post-processing. It considers the release of census data and 
examines, both theoretically and empirically, the behavior of a 
widely adopted class of post-processing functions.
\end{abstract}

\section{Introduction}
\label{sec:Introduction}
Data sets and statistics about groups of individuals are increasingly collected and released, feeding many optimization and learning algorithms. 
In many cases, the released data contain sensitive information whose privacy is strictly regulated. For example, in the U.S., the census data is regulated under Title 13 \cite{Title13}, which requires that no individual be identified from any data release by the Census Bureau. In Europe, data release are regulated according to the General Data Protection Regulation \cite{GDPR}, which addresses the control and transfer of personal data. Hence statistical agencies release privacy-preserving data and statistics that conform to these requirements.

\emph{Differential Privacy} \cite{Dwork:06} is of particular interest to meet this goal. Differential privacy is a formal privacy definition that bounds the disclosure risk of any individual participating in a computation. It is  considered the de-facto standard for privacy protection and has been 
adopted by various corporations \cite{erlingsson2014rappor,apple}  
and governmental agencies \cite{abowd2018us}.

On data-release tasks, differentially private algorithms, typically, inject carefully calibrated noise to the data before release. However, whereas this process guarantees
privacy, it also affects the fidelity  of the released data. In
particular, the injected noise often produces data sets that violate
consistency constraints of the application domain. For example, in census statistics, the number of people satisfying a property must be consistent in a geographical hierarchy, e.g., at the national, state, and county levels. The injection of independent noise, however, cannot ensure the consistency of these constraints.

To overcome this limitation, the differentially private outputs can be \emph{post-processed} via data-independent functions that 
transform the noisy data to render it consistent with the domain constraints.
The post-processing step is guaranteed to retain differential privacy. Moreover, when the feasible region is convex, a largely adopted class of post-processing functions, called \emph{projections}, is guaranteed to improve accuracy. Post-processing has been applied successfully in many applications, including census data \cite{abowd2018us}, energy systems \cite{fioretto:TSG20}, and mobility \cite{xi:15}. However, the effect of post-processing on the noise distribution is poorly understood: It is often argued that it may introduce bias and/or increase variance. Figure \ref{fig:bias} illustrates this aspect on a census data-release problem, described later in the paper. It depicts the distribution of the Laplacian residual $\tilde{\bm{x}} - \bm{x}$, where $\bm{x}$ denotes the true data and $\tilde{\bm{x}}$ the noisy data, obtained by applying Laplacian noise to $\bm{x}$, as well as the post-processed residual $\hat{\bm{x}} - \bm{x}$, where $\hat{\bm{x}}$ is the projection of $\tilde{\bm{x}}$ onto the feasible region. The results are shown for two counties, and, as can be seen, the post-processing introduces significant bias on their associated privacy-preserving data.

\begin{figure}
\centering
\includegraphics[width=0.9\linewidth]{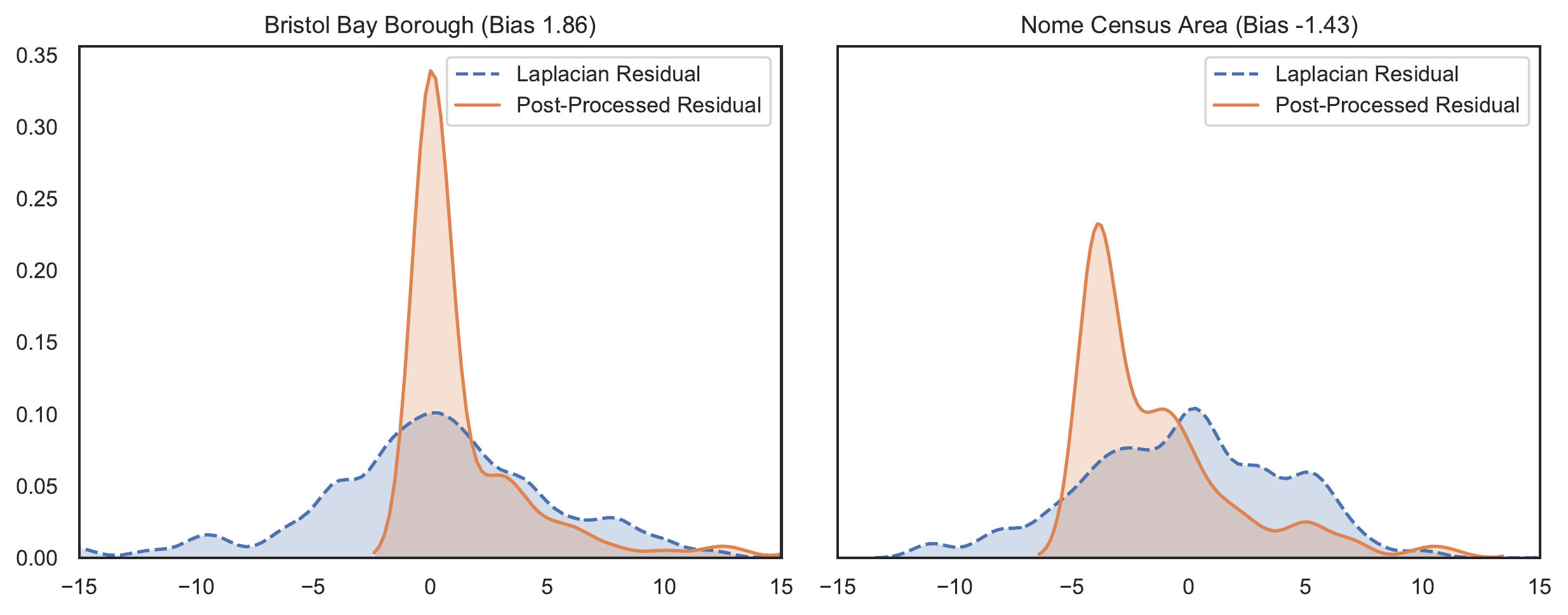}
\caption{Bias of Post-Processing on the Census Problem.}
\label{fig:bias}
\vspace{-0.5cm}
\end{figure}

The key contribution of this paper is to take a first step towards understanding the properties of post-processing. Motivated by census applications, it studies the behavior of two widely adopted classes of post-processing functions, called \emph{projections}, for domains where the feasibility space is specified by linear equations. The two classes differ by the presence of non-negativity constraints. The paper shows that, when non-negativity constraints are absent, the projection does not introduce bias. When projections include non-negativity constraints, the paper presents an upper bound on the bias, which provides some insights on the type of problems for which the bias will be significant. 
Finally, the paper provides a detailed analysis of an important sub-problem used to satisfy hierarchical constraints in data-release tasks: It fully characterizes the residual distribution of the post-processed data, shows that it converges towards the Laplace distribution, and shed some interesting light on the effect of projections on the variance of the post-processed data, which may have strong implications with respect to group fairness.

\section{Related Work}
\label{sec:Related_Work}
The adoption of post-processing to ensure that differentially private output satisfy some property of interest is commonly adopted in the privacy literature. 
Important contributions include the \emph{hierarchical mechanism} of \cite{hay:10} and its extensions \cite{qardaji:13,cormode:12}, which uses a post-processing step that enforces additive constraints based on a tree structure of the data universe to answer count queries over ranges. Other methods have incorporated a partitioning scheme to the data-release problem to increase the accuracy of the privacy-preserving data by cleverly splitting the privacy budget in different hierarchical levels \cite{xiao2010differentially,cormode2012differentially,zhang2016privtree}. 

These post-processing algorithms have been used to release privacy preserving data sets for a wide array of applications, including transportation \cite{fioretto:AAMAS-18}, location privacy \cite{xi:15}, and energy optimization \cite{fioretto:TSG20}. Of particular interest is the TopDown algorithm \cite{abowd2018us}, used by the US Census for the 2018 end-to-end test in preparation for the 2020 release. The algorithm is based on post-processing to satisfy consistency of hierarchical counts. 


\section{Preliminaries: Differential Privacy}
\label{sec:Differential_Privacy}
\newcommand{\Lap}{\mathrm{Lap}}

\emph{Differential privacy} (DP) \cite{Dwork:06} is a rigorous privacy notion used to protect disclosures of an individual's data in a computation. 
Informally, it states that the probability of any differentially private output does not change much when a single individual data is added or removed to the data set, limiting the amount of information that the output reveals about any individual. 

\begin{definition}[Differential Privacy]
A randomized mechanism $\cM \!:\! \cD \!\to\! \cR$ with
domain $\cD$ and range $\cR$ is $\epsilon$-differentially private if, for any output $O \subseteq \cR$ and data sets $D, D' \in \cD$ differing by at most one entry (written $D \sim D'$),
	\begin{equation}
  	\label{eq:dp_def}
  	\Pr[\cM(D) \in O] \leq \exp(\epsilon)\, Pr[\cM(D') \in O].
	\end{equation}
\end{definition}
\noindent 
The parameter $\epsilon \geq 0$ is the \emph{privacy loss} of the
mechanism, with values close to $0$ denoting strong privacy. 

An important differential privacy property is its \emph{immunity to post-processing}, stating that a differentially private output can be arbitrarily transformed, using some data-independent function, without impacting its privacy guarantees.
\begin{theorem}[Post-Processing \cite{Dwork:06}] 
	\label{th:postprocessing} 
	Let $\cM$ be an $\epsilon$-differentially private mechanism and $g$ be an arbitrary mapping from the set of possible outputs to an arbitrary set. Then, $g \circ \cM$ \mbox{is $\epsilon$-differentially private.}
\end{theorem}

A function $f$ (also called \emph{query}) from a data set $D \in \cD$ to a result set $R \subseteq \RR^n$ can be made differentially private by injecting random noise to its output. 
The amount of noise depends on the \emph{global sensitivity} of the query, denoted by $\Delta_f$ and defined as
\(
	\Delta_f = \max_{D \sim D'} \left\| f(D) - f(D')\right\|_1.
\)

The Laplace distribution with $0$ mean and scale $\lambda$, denoted by $\Lap(\lambda)$, has a probability density function $\Lap(x|\lambda) = \frac{1}{2\lambda}e^{-\frac{|x|}{\lambda}}$. It can be used to obtain an $\epsilon$-differentially private algorithm to answer numeric queries \cite{Dwork:06}. 
In the following, $\Lap(\lambda)^n$ denotes the i.i.d.~Laplace distribution with 0 mean and scale $\lambda$ over $n$ dimensions. 
\begin{theorem}[Laplace Mechanism]
	\label{th:m_lap} 
	Let $f: \cD \to \RR^n$ be a numeric query. The Laplace mechanism that outputs $f(D) + \bm{\eta}$, where $\bm{\eta} \in \RR^n$ is drawn from the Laplace distribution $\Lap(\nicefrac{\Delta_f}{\epsilon})^n$, achieves $\epsilon$-differential privacy.
\end{theorem}
\noindent

\section{Settings and Goal}
\label{sec:Settings_and_Goal}
\begin{figure}
\centering
\includegraphics[width=0.85\linewidth]{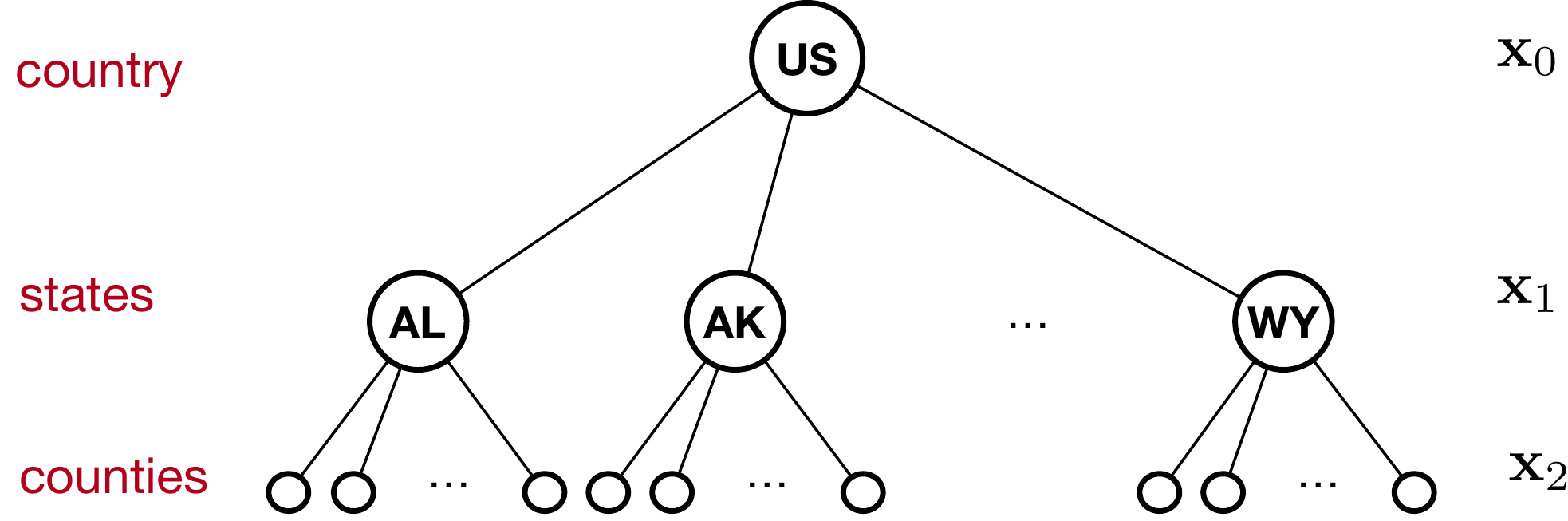}
\caption{Example of hierarchical data set. \label{fig:example}}
\end{figure}

The paper uses the following notation: boldface symbols denote vectors while italic symbols are used to denote scalars or random variables. 
The paper considers data sets of the form $\bm{x} \in \RR^n$, where each element $x_i$ of $\bm{x}$ is a real-valued quantity describing, for example, the number of individuals living in a geographical region. To produce  $\epsilon$-differentially private outputs, this work adopts the Laplace mechanism which, for an appropriately chosen $\lambda$, produces a new privacy-preserving data set $\tilde{\bm{x}} = \bm{x} + \Lap(\lambda)^n$. However, all results presented in this paper generalizes to other symmetric distributions as discussed later. 

The original data $\bm{x}$ is assumed to satisfy a set of data independent constraints $\cK$. This paper focuses on the case where $\cK$ is a set of linear constraints which, as mentioned in the introduction, arise in a widespread number of applications \cite{fioretto2019differential,abowd2018us,xi:15,fioretto:TSG20}. 
Of particular relevance to this work are \emph{hierarchical data release problems}, as those faced by the US Census Bureau. Consider the illustration in Figure \ref{fig:example}. The tree depicts the hierarchy of the US territories, partitioned in states and counties. Each node is associated with a value representing the number of individuals living in the corresponding territory. The constraint set $\cK$ then specifies that the value of a node is the sum of the values of its children and that all values are non-negative.

Due to the use of independent noise, the differentially private version $\tilde{\bm{x}}$ of $\bm{x}$, may not satisfy the original constraints. This scenario happens with very high probability in the hierarchical data-release problem considered. The paper, thus, focuses on mechanisms that generates outputs $\hat{\bm{x}}$ that satisfy two properties: (1) they guarantee $\epsilon$-differential privacy, and (2) $\hat{\bm{x}}$ satisfy the constraints in $\cK$.  

\subsection*{Projection Operators}

To meet these two objectives, the paper studies an important class of post-processing operators, called \emph{projections}, that transform released data $\tilde{\bm{x}}$ to satisfy the constraints in $\cK$. This paper focuses on the following two projections:

\begin{equation}
    \begin{aligned}
        \postdata\coloneqq~&\underset{\bm{v}\in \cK}{\arg\min}
        \norm{\bm{v}-\noisydata}_2 \\
        \cK &= \{\bm{v} \in \mathbb{R}^n \mid 
        A\bm{v} = \bm{b}\}\!\!\!
    \end{aligned}\tag{$P$} \label{program}
\end{equation}
and,
\begin{equation}
    \begin{aligned}
        \nnpostdata
        \coloneqq~&\underset{\bm{v}\in \cK}{\arg\min}
        \norm{\bm{v}-\noisydata}_2 \\
        \cK &= \{\bm{v} \in \mathbb{R}^n \mid  
        A\bm{v} = \bm{b}; \bm{v} \geq \bm{0}\}
    \end{aligned}\tag{$\nnprogram$}\!\!\!\!\!\! \label{nnprogram}
\end{equation}

\noindent 
where $\tilde{\bm{x}}$ is the privacy-preserving input to the projection operators, obtained by applying the Laplace mechanism to $\bm{x}$, $A$ is an $m\times n$ matrix, and $\bm{b}$ is an $m$-dimensional vector. $A$ and $\bm{b}$ are assumed to be public, non-sensitive information in this paper. By the post-processing immunity of differential privacy (Theorem \ref{th:postprocessing}) the projections operators \eqref{program} and \eqref{nnprogram} satisfy differential privacy. Both optimizations find a feasible solution that minimizes the $l_2$-distance to the noisy data $\noisydata$. The existence and uniqueness of their solutions are guaranteed by convexity. 
These programs have been adopted by a vast array of applications. In particular, the census hierarchical data-release problem, analyzed in this paper as a case study, restores consistency of the hierarchical constraints  using an instance of problem \eqref{nnprogram}.

The theoretical results in this paper are illustrated using an empirical analysis from this census case study. For each instance associated with the true counts $\truedata$, 
the noise $\noise$ is i.i.d. drawn from the double-sided geometric distribution $\noise\sim \text{Geom}(\nicefrac{\Delta_f}{\epsilon})^n$, i.e., the discrete analog to the Laplace distribution. The results in this paper are generally presented for continuous distributions but they carry over naturally to this geometric distribution. The privacy budget $\epsilon$ is set to be $0.5$ and the experiments perform 100 independent runs.

\section{Analysis of Bias in Post-Processing}

\subsection{Bias of Program \eqref{program}}

This section studies the bias induced by program \eqref{program}, when the noisy data $\noisydata$ is obtained by applying noise drawn from a symmetric probability distribution. Recall that a distribution with probability density function $f$ is symmetric if there exists a value  $x_{0}$ such that $f(x_{0}-\delta )=f(x_{0}+\delta)$ for all $\delta$. This is the case of the Laplace and the double-sided geometric distributions. This section relies on the concept of a reflection operators. 

\begin{definition}[Reflection operator]
    The operator $\refopt_{\bm{v}}(\cdot)$ is said to be a reflection
    operator across the vector $\bm{v}\in\RR^n$ if,
    for any $\bm{u}\in\RR^n$, the following identity holds:
    \begin{equation*}
        \refopt_{\bm{v}}(\bm{u}) = 2\bm{v} - \bm{u}\,.
    \end{equation*}
\end{definition}

\begin{lemma}\label{lem:err_sum_0}
    The reflection operator $\refopt_{\truedata}$ and $\postdata$ (as an operator) are commutative, i.e.,
    \begin{equation}\label{eq:lem1}
        \refopt_{\truedata}\left(\postdata(\noisydata)\right)=\postdata\left(\refopt_{\truedata}(\noisydata)\right)\,.
    \end{equation}
\end{lemma}
\begin{proof}
   The right hand side of \eqref{eq:lem1} is given by
    \begin{equation*}
        \begin{aligned}
        \bm{x}'\coloneqq\underset{\bm{v}\in\RR^n}{\arg\min}~& 
        \norm{\bm{v}-\refopt_{\truedata}(\noisydata)}_2 \\
        \text{s.t.}~ & A\bm{v}=\bm{b}\,,
    \end{aligned}
    \end{equation*}
    where $\bm{x}'$ is a shorthand for the solution $\postdata(\refopt_{\truedata}(\noisydata))$.
    By reflection, $\refopt_{\truedata}(\bm{x}')$ is a solution to the  optimization problem:
    \begin{equation*}
        \begin{aligned}
            \refopt_{\truedata}(\bm{x}')=\underset{\bm{v}\in\RR^n}{\arg\min}~& 
            \norm{\refopt_{\truedata}({\bm{v}})-\refopt_{\truedata}(\noisydata)}_2 \\
            \text{s.t.}~ & A\refopt_{\truedata}({\bm{v}})=\bm{b}\,.
        \end{aligned}
    \end{equation*}
    By the definition of the reflection operator and the feasibility of the true data, we have that
     \begin{align}
        &\norm{\refopt_{\truedata}({\bm{v}})-\refopt_{\truedata}(\noisydata)}_2
        = \norm{\bm{v}-\noisydata}_2\nonumber\,,\\
        &A\bm{v}=A(2\truedata-\refopt_{\truedata}({\bm{v}}))=2\bm{b}-\bm{b}=\bm{b}
    \end{align}
    and the previous optimization problem is equivalent to \eqref{program}:
    \begin{equation*}
        \begin{aligned}
            \refopt_{\truedata}(\bm{x}')=\underset{\bm{v}\in\RR^n}{\arg\min}~& 
            \norm{{\bm{v}}-\noisydata}_2 \\
            \text{s.t.}~ & A{\bm{v}}=\bm{b},
        \end{aligned}
    \end{equation*}
    because
    since \eqref{program} is convex, $\refopt_{\truedata}({\bm{x}'})=\postdata(\noisydata)$. By applying the reflection operator on both sides, $\bm{x}'=\postdata(\refopt_{\truedata}(\noisydata))=\refopt_{\truedata}(\postdata(\noisydata)).$
\end{proof}

\begin{figure}[t]
    \centering
     \resizebox{0.6\linewidth}{!}{%
     \begin{tikzpicture}
        \draw[->] (-0.3,0) -- (7,0) node[right] {$v_1$};
        \draw[->] (0,-0.4) -- (0,3.5) node[above] {$v_2$};
        \draw[thick] (-0.3, 3.15) -- (6.6, -0.3);
        
        \coordinate (truedata) at (3, 1.5);
        \coordinate (noisydata) at (1.7, 3.4);
        \coordinate (refnoisydata) at (4.3, -0.4);
        \coordinate (postdata) at (1.2, 2.4);
        \coordinate (refpostdata) at (4.8, 0.6);
        
        \fill (truedata) circle (2.5pt) node[right] {$~\truedata$};
        \fill (noisydata) circle (2.5pt) node[right] {$\noisydata$};
        \fill (refnoisydata) circle (2.5pt) node[right] {$\refopt_{\truedata}(\noisydata)$};
        \fill (postdata) circle (2.5pt) node[below] {$\postdata(\noisydata)$} ;
        \fill (refpostdata) circle (2.5pt) node[right] {$~\postdata(\refopt_{\truedata}(\noisydata))$} ;
        
        \draw[dashed] (noisydata) -- (postdata);
        \draw[dashed] (refnoisydata) -- (refpostdata);
        \draw[dashed] (noisydata) -- (refnoisydata);
    \end{tikzpicture}
    }
    \caption{Illustration of Lemma \ref{lem:err_sum_0} and Corollary \ref{cor:err_sum_0}.}\label{fig:err_sum_0}
\end{figure}
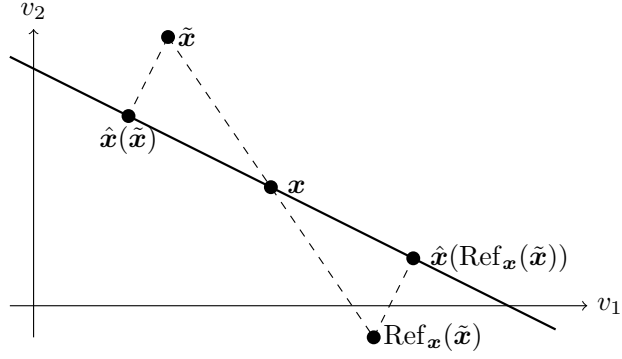

\noindent
Figure \ref{fig:err_sum_0} illustrates Lemma \ref{lem:err_sum_0}: It shows that the true data $\truedata$ is the midpoint between the post-processed solutions associated with the noisy data $\noisydata$ and its reflection. 

Let $\err{\bm{y}}=\bm{y} - \truedata$, i.e., the entrywise difference between $\bm{y}$ and the true data. 
\begin{corollary}\label{cor:err_sum_0}
    The errors associated with the noisy data $\noisydata$ and its reflection
    $\refopt_{\truedata}(\noisydata)$ sums to $\bm{0}$, i.e.,
    \begin{equation*}
        \err{\postdata(\noisydata)}+\err{\postdata(\refopt_{\truedata}(\noisydata))}=\bm{0}\,.
    \end{equation*}
\end{corollary}
\begin{proof}
By Lemma \ref{lem:err_sum_0}, 
    \begin{eqnarray*}
    &&\err{\postdata(\noisydata)}+\err{\postdata(\refopt_{\truedata}(\noisydata))}\\
        &=&\err{\postdata(\noisydata)}+\err{\refopt_{\truedata}\left(\postdata(\noisydata)\right)}\\
        &=&\left(\postdata(\noisydata)-\truedata\right) + \left(\refopt_{\truedata}\left(\postdata(\noisydata)\right)-\truedata\right) = \bm{0}.
    \end{eqnarray*}
\end{proof}

\noindent 
The following theorem is a positive result: It shows that program \eqref{program} does not introduce bias.

\begin{theorem}\label{thm:no_bias}
     Program \eqref{program} does not introduce bias, i.e.,
    \begin{equation*}
       \bias{\postdata(\noisydata)}\coloneqq\EE{\noisydata}{\err{\postdata(\noisydata)}}=\bm{0}\,,
    \end{equation*}
    where the expectation is taken over the distribution of the noisy data $\noisydata$.
    In other words, the post-processed solution $\postdata$ (as a random vector) is \emph{unbiased}.
\end{theorem}

\begin{proof}
    Let $f_{\noisydata}$ denote the probability density function of the noisy data $\noisydata$,
    which is symmetric with respect to the true data $\truedata$.
    Then, the expectation of the resulting error is computed as follows.
    \begin{align}
        &\EE{\noisydata}{\err{\postdata(\noisydata)}}\nonumber
        \!=\!\int_{\bm{y}\in\RR^n}\!\!\!\!\!\!\!\!\!
        \err{\postdata(\bm{y})}\cdot f_{\noisydata}(\bm{y})d\bm{y}\nonumber\\
        &=\frac{1}{2}\int_{\bm{y}\in\RR^n}\!\!\!\!\!\!\!\!\!
        \err{\postdata(\bm{y})}\cdot f_{\noisydata}(\bm{y})d\bm{y}
        ~+\nonumber \\
        &\hspace{14pt}
        \frac{1}{2}\int_{\bm{y}\in\RR^n}\!\!\!\!\!\!\!\!\!
        \err{\postdata(\refopt_{\bm{x}}(\bm{y}))}\cdot
        f_{\noisydata}(\refopt_{\bm{x}}(\bm{y}))d\bm{y}
        \nonumber\\
        &=\frac{1}{2}\int_{\bm{y}\in\RR^n}\!\!\!\!\!\!\!\!\!
        \left[\err{\postdata(\bm{y})}+\err{\postdata(\refopt_{\bm{x}}(\bm{y}))}\right]\cdot
        f_{\noisydata}(\bm{y})d\bm{y}\label{eq:no_bias_1}\\
        &=\frac{1}{2}\int_{\bm{y}\in\RR^n}\!\!\!\!\!\!\!\!\!
        \bm{0}\cdot
        f_{\noisydata}(\bm{y})d\bm{y}\label{eq:no_bias_2}\\
        &=\bm{0}\,,\nonumber
    \end{align}
    where Equation \eqref{eq:no_bias_1} comes from the symmetric distribution of the noisy data $\noisydata$, i.e., for any
    $\bm{y}\in\RR^n$,
    \begin{equation*}
        f_{\noisydata}(\bm{y})=f_{\noisydata}(\truedata+(\bm{y}-\truedata))=f_{\noisydata}(\truedata-(\bm{y}-\truedata))=f_{\noisydata}(\refopt_{\bm{x}}(\bm{y})),
    \end{equation*}
    and Equation
    \eqref{eq:no_bias_2} is due to Corollary \ref{cor:err_sum_0}.
\end{proof}


\subsection{Bias of Program \eqref{nnprogram}}

Theorem \ref{thm:no_bias} indicates that the bias in program \eqref{nnprogram} comes from the non-negativity constraints. The section bounds this bias by leveraging the insights of Theorem \ref{thm:no_bias}. It assumes that the feasible region $\nnregion$ is bounded (which holds, in many practical setting, including in the census data release case) and that the noisy data $\noisydata$ is the output of the Laplace mechanism applied to the true data $\truedata$, i.e., $\noisydata=\truedata+\text{Lap}(\lambda)^n$. It will leverage the prior positive results by isolating a subset of the feasible space close under refection. The first lemma computes the probability that the Laplace mechanism produces an output in a ball of radius $r$. The proof is by induction on the dimension $n$.

\begin{lemma}\label{lem:l1_ball}
    Given a random vector $\noise = [\eta_1,\dots,\eta_n]$, where $\{\eta_i\}_{i\in[n]}$ are i.i.d.
    random variables drawn from a Laplace distribution ${\rm Lap}(\lambda)$ ($\lambda >0$),
     the following identity holds for any $r\geq 0$:
    \begin{equation}\label{eq:prob_comp}
        \pr{\noise\in \BB{r}{\bm{0}}} = 1-\exp\left(\frac{-r}{\lambda}\right)\cdot \sum_{i=0}^{n-1}\frac{r^i}{i!\cdot\lambda^i}\,,
    \end{equation}
    where $\BB{r}{\bm{0}}$ is the $l_1$-ball of radius $r$ centered at $\bm{0}$, i.e.,
    \begin{equation*}
        \BB{r}{\bm{0}}=\left\{\bm{v}\in \RR^n\mid \norm{\bm{v}}_1\leq r\right\}\,.
    \end{equation*}
\end{lemma}

\noindent
A similar result can be obtained for the double-sided geometric distribution. If the noisy data follows a $\text{Geom}(a)^n$ distribution, then 
\begin{equation*}
        \pr{\noise\in \BB{r}{\bm{0}}} = 1-\frac{2a^{r+1}}{1+a} \sum_{i=0}^{n-1} h_i(r)\cdot \left(\frac{1-a}{1+a}\right)^i,
    \end{equation*}
where $\{h_i\}_{i\in \NN}$ is a family of polynomials with $h_0(r)=1$ and $h_{i+1}(r)=\sum_{v=-r}^r h_i(r-\vert v\vert)$ for any
$i\in \NN$.
\noindent
The rest of this section is presented in terms of the Laplace distribution but the results can be generalized to any distribution satisfying a version of Lemma \ref{lem:l1_ball}.

\begin{corollary}\label{col:prob_bound}
    Suppose that the noisy data $\noisydata$ is the output of the Laplace mechanism, i.e., $\noisydata = \truedata+{\rm Lap}(\lambda)^n$ with $\lambda >0$. Then,
    for any $r\geq 0$,
    \begin{equation*}
        \pr{\noisydata\in \BB{r}{\truedata}}=1-\exp\left(\frac{-r}{\lambda}\right)\cdot \sum_{i=0}^{n-1}\frac{r^{i}}{i!\cdot\lambda^{i}}.
    \end{equation*}
\end{corollary}
\begin{proof}
    Let $\noise$ denote the $n$-dimensional random vector of the Laplacian noise added to the true data
    $\truedata$, i.e., $\noise=\noisydata - \truedata$. By the definition of the Laplace mechanism,
    $\noise=[\eta_1,\dots,\eta_n]$ consists of $n$ i.i.d. components, each of which is drawn from the Laplace
    distribution $\text{Lap}(\lambda)$. Then, by Lemma \ref{lem:l1_ball}, 
    \begin{align*}
        \Pr(\noise &\in \BB{r}{\bm{0}}) = \pr{\noisydata-\truedata\in \BB{r}{\bm{0}}}\\
        &=1-\exp\left(\frac{-r}{\lambda}\right)\cdot \sum_{i=0}^{n-1}\frac{r^i}{i!\cdot\lambda^i}\,,\qquad\forall~r\geq 0\,.
    \end{align*}
    Since $\noisydata-\truedata\in \BB{r}{\bm{0}}$ iff 
    $\noisydata\in\BB{r}{\truedata}$, for any $r\geq 0$,
    \begin{equation*}
        \pr{\noisydata\in \BB{r}{\truedata}}=1-\exp\left(\frac{-r}{\lambda}\right)\cdot \sum_{i=0}^{n-1}\frac{r^{i}}{i!\cdot\lambda^{i}}.
    \end{equation*}
\end{proof}

\noindent
Let $\radius$ be $\min_{i\in[n]}\bm{x}_i$. The next lemma states that $\BB{\radius}{\truedata}$ is a feasible subspace where there is no bias.

\begin{figure*}[!th]
\centering
\includegraphics[width=0.4\linewidth]{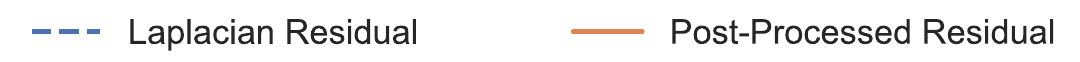}\\
\includegraphics[width=0.85\linewidth]{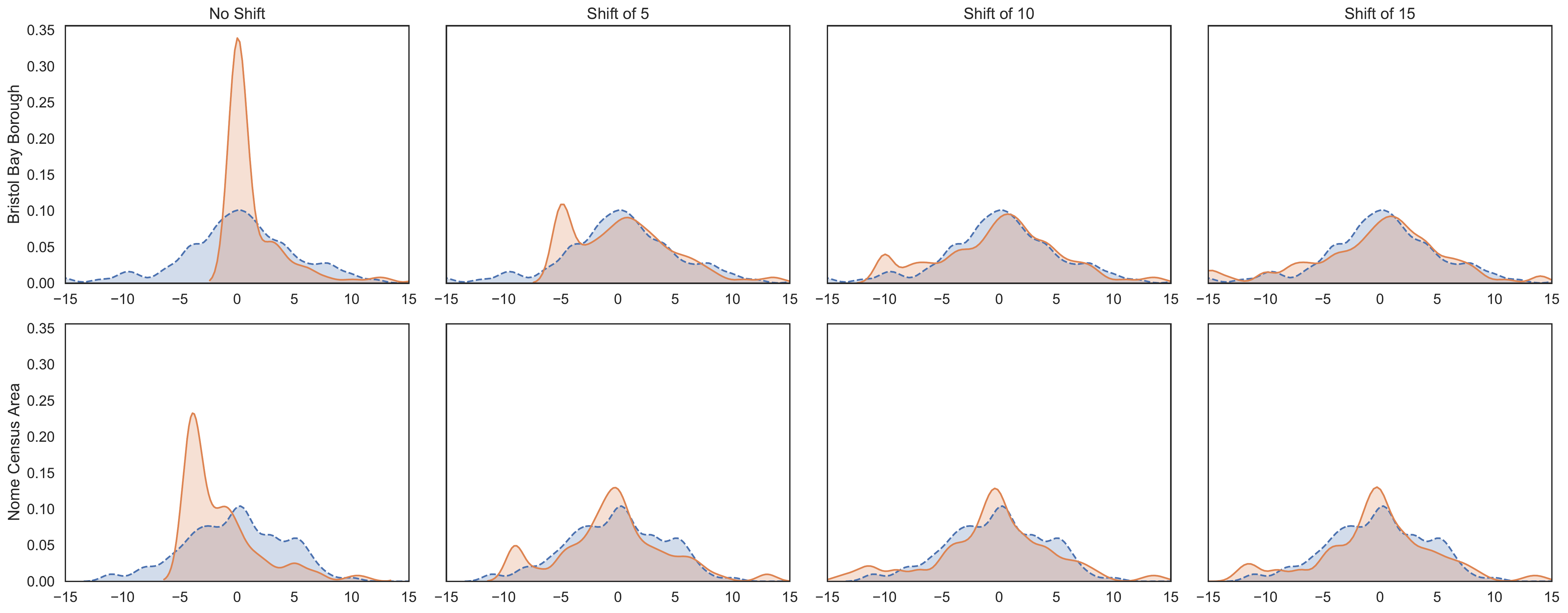}
\caption{Bias of Post-Processing on the Census Problem as $r_m$ Increases.}
\label{fig:biasnn}
\end{figure*}

\begin{lemma}\label{lem:post_is_posi}
    For any noisy data $\noisydata\in \BB{\radius}{\truedata}$, the post-processed solution $\postdata$ of program \eqref{program} is non-negative and 
    equal to solution $\nnpostdata$ of program \eqref{nnprogram}.
\end{lemma}
\begin{proof}
    Since $\noisydata$ belongs to the $l_1$-ball $\BB{\radius}{\truedata}$,  $\norm{\noisydata-\truedata}_2$ is also bounded from the above by $\radius$ since
    \begin{equation*}
        \norm{\noisydata-\truedata}_2\leq \norm{\noisydata-\truedata}_1\leq \radius\,.
    \end{equation*}
    By convexity of $\region$, $\norm{\postdata-\truedata}_{\infty}\leq \norm{\postdata-\truedata}_2\leq \radius$. Moreover, $\postdata$ is non-negative since its $l_\infty$-distance to $\truedata$ is bounded by $\radius$ and the result follows by optimality of $\nnpostdata$.
\end{proof}


\noindent The next theorem is the main result of this section and it bounds the bias of program \eqref{nnprogram}.

\begin{theorem}
\label{thm:biasn}
    Suppose that the noisy data $\noisydata$ is the output of the Laplace mechanism with scale $\lambda$.
    The bias of the post-processed solution $\nnpostdata$ of  program \eqref{nnprogram} is bounded, in
    $l_{\infty}$ norm, by
    \begin{align*}
        &\norm{\bias{\nnpostdata(\noisydata)}}_{\infty}=\norm{\EE{\noisydata}{\err{\nnpostdata(\noisydata)}}}_{\infty}\\
        \leq &~\bound\cdot \exp\left(\frac{-\radius}{\lambda}\right)\cdot \sum_{i=0}^{n-1}\frac{(\radius)^i}{i!\cdot\lambda^i}\,,
    \end{align*}
    where $\bound$ represents the value $\sup_{\bm{v}\in \nnregion}\norm{\bm{v}-\truedata}_\infty$,
    which is finite due to the boundedness of the feasible region $\nnregion$.
\end{theorem}
\begin{proof}
    \begin{align*}
        &\bias{\nnpostdata(\noisydata)}=\EE{\noisydata}{\err{\nnpostdata(\noisydata)}}\\
        =&\EE{\noisydata}{\err{\nnpostdata(\noisydata)}\mid \noisydata\in \BB{\radius}{\truedata}}\cdot
        \pr{\noisydata\in\BB{\radius}{\truedata}} +\\
        &\EE{\noisydata}{\err{\nnpostdata(\noisydata)}\mid \noisydata\notin \BB{\radius}{\truedata}}\cdot
        \pr{\noisydata\notin\BB{\radius}{\truedata}}\,.
    \end{align*}
    By Lemma \ref{lem:post_is_posi} and Theorem \ref{thm:no_bias}, the left-hand side of the sum is zero. As a result,
    \begin{align}
        &~\norm{\bias{\nnpostdata(\noisydata)}}_{\infty}\nonumber\\
        =&~\norm{\EE{\noisydata}{\err{\nnpostdata(\noisydata)}\mid \noisydata\notin \BB{\radius}{\truedata}}}_\infty\cdot\pr{\noisydata\notin\BB{\radius}{\truedata}}\nonumber\\
        \leq &~\EE{\noisydata}{\norm{\err{\nnpostdata(\noisydata)}}_\infty\mid \noisydata\notin \BB{\radius}{\truedata}}\cdot\pr{\noisydata\notin\BB{\radius}{\truedata}}\nonumber\\
        \leq &~\bound \cdot \pr{\noisydata\notin\BB{\radius}{\truedata}}\label{eq:bias_bound_1}\\
        = &~\bound \cdot \exp\left(\frac{-\radius}{\lambda}\right)\cdot \sum_{i=0}^{n-1}\frac{(\radius)^i}{i!\cdot\lambda^i}\label{eq:bias_bound_2}\,, 
    \end{align}
    where \eqref{eq:bias_bound_1} follows from the feasibility of  $\nnpostdata(\noisydata)$ and 
    \begin{align*}
        &\norm{\err{\nnpostdata(\noisydata)}}_\infty = \norm{\nnpostdata(\noisydata)-\truedata}_\infty\\
        \leq & \sup_{\bm{v}\in \nnregion}\norm{\bm{v}-\truedata}_\infty = \bound,
    \end{align*}
    since the feasible region is bounded by hypothesis. Equation \eqref{eq:bias_bound_2} follows from Corollary \ref{col:prob_bound}.
\end{proof}

\noindent
Figure \ref{fig:biasnn} illustrates Theorem \ref{thm:biasn}. It reports the same residuals as in Figure \ref{fig:bias} but with the true county counts increased by a positive shift factor. This increases the value of $r_m$ and the bias progressively disappears as $r_m$ grows. This observation can give  insights to statistical agencies about what can be reported without introducing significant bias, informing their decisions on the granularity of the data releases.

To complement these results, the theoretical bound is also compared on the post-processing of New Mexico and its counties. The state has a population of 7,289,112, 33 counties, $r_m = 348$, and the experiment uses $\lambda = 5$. The theoretical bound is 0.29, while the empirical bias is 0.06. The results may not be as tight for larger states, since the bound depends on $C'$, the maximum distance between the real data and a point in the feasible space. 

\section{Analysis of Fairness in Projections}

This section provides a detailed analysis of the distribution of the post-processed 
noise for a special case of program \eqref{program} defined as follows:
\begin{equation}
    \begin{aligned}
        \sumpostdata \coloneqq~&\underset{\bm{v}\in \cK}{\arg\min} \norm{\bm{v}-\noisydata}_2 \\
        \sumregion&=\left\{\bm{v}\in\RR^n~\bigg{|}~\sum_{i=1}^n v_i = b\right\},
    \end{aligned}\tag{$P_{\text{S}}$} \label{sumprogram}
\end{equation}
where $b\in\RR$ is a constant. This specific post-processing step \eqref{sumprogram} requires that the components of its output should be summed up to the constant $b$, which makes it broadly applicable. For instance, in the census context, program \eqref{sumprogram}  makes sure that the state populations are compatible with the overall US population, which is viewed as public information. Similar post-processing steps take place at the state level. The section will reveal an interesting connection between the post-processing step and the census model itself. 

The next theorem is a key result: it characterizes the marginal distribution of the post-processed noise  $\sumpostdata-\truedata$, i.e., the distribution of  $\err{\sumpostdata}_i=\sumpostdatai-\truedata_i$ for any $i\in[n]$. It is expressed in terms of the Laplace distribution but again generalizes to other distributions. 
\begin{theorem}\label{lem:err_dist}
    Let $\{\eta_i\}_{i\in[n]}$ be $n$  i.i.d. random variables drawn from a Laplace distribution ${\rm Lap}(\lambda)$.
    The marginal error of the post-processed solution $\sumpostdata$ of  program \eqref{sumprogram} follows the distribution:
    \begin{equation*}
        \err{\sumpostdata}_i=\sumpostdatai-\truedata_i\sim \frac{(n-1)\eta_i-\sum_{j\neq i}\eta_i}{n}\,\;\;\forall~i\in[n],
    \end{equation*}
    with variance
    \begin{align*}
        \var{\err{\sumpostdata}_i}=&\left(1-\frac{1}{n}\right)\cdot \var{{\rm Lap}(\lambda)}\\
        =&~2\lambda^2\left(1-\frac{1}{n}\right)\,,&&\forall~i\in[n].
    \end{align*}
\end{theorem}
 \begin{proof} (Sketch)
    Without the loss of generality, the proof considers $\err{\sumpostdata}_1$.
    Let $\{\unit_i\}_{i\in[n]}$ be the standard basis of the $n$-dimensional space $\RR^n$ such that 
    the noise $\noise$ added to the true data $\truedata$ can be represented as $\noise = \sum_{i=1}^n \eta_i\cdot \unit_i$
    where $\{\eta_i\}_{i\in[n]}$ are $n$  i.i.d. random variables drawn from a Laplace distribution $\text{Lap}(\lambda)$ and $\{\unit_i\}_{i\in[n]}$ be the standard basis of $\RR^n$. Consider the probability density of the marginal error $\err{\sumpostdata}_1$ at $v$, i.e., the integration of the original Laplacian noise over a set $S_v$ as follows.
    \begin{align*}
        f_{\err{\sumpostdata}_1}(v)=\int_{\bm{y}\in S_v} \frac{1}{(2\lambda)^n}\exp\left(-\frac{\norm{\bm{y}-\truedata}_1}{\lambda}\right)d\bm{y}\,,
    \end{align*}
    where $S_v = \left\{\bm{u}\mid \sumpostdata(\bm{u})_1-x_1=v\right\}$. To compute this integration, it is easier to exploit the definition of the projection operator and express $S_v$ differently through a basis transformation. Indeed, it can be shown that $S_v$ can be rewritten as 
     Given the new basis $\{\newunit_i\}_{i\in[n]}$, the set $S_v$ can be expressed in the following form
    \begin{equation*}
        S_v=\left\{\bm{u}=\truedata+\sum_{i=1}^n c'_i\cdot \newunit_i~\bigg{|}~ c'_n=v\cdot\sqrt{\frac{n}{n-1}}\right\}
    \end{equation*}
    where $\{\newunit_i\}_{i\in[n]}$ is an orthonormal basis of $\RR^n$ with 
$
        \newunit_1=\frac{1}{\sqrt{n}}[1, \ldots, 1]^T
$
    and $\newunit_n$ has $\sqrt{(n-1)/n}$ as
    its first component and $-\sqrt{1/(n(n-1))}$ as the remaining ones. The marginal probability density distribution $\err{\sumpostdata}_1$ is then given by
    \begin{align*}
        &\int_{\bm{y}'_{-n}\in \RR^{n-1}} f'_{\noise}((\bm{y}'_{-n},v\cdot\sqrt{\frac{n}{n-1}})) d\bm{y}'_{-n} \label{eq:margin_dist_2}
    \end{align*}
    where $\bm{y} =[y_1,\dots,y_n]^\top$, $\bm{y}'_{-n} =[y'_1,\dots,y'_{n-1}]^\top$, and $f'_{\noise}$ represents the Laplace probability density function distribution under the new basis $\{\newunit_i\}_{i\in[n]}$. It comes that the random variable $\eta'_n$, i.e., the noise $\eta_n$ in the new basis, shares the same distribution with
    \begin{equation*}
        \inner{\newunit_n}{\noise}=\sum_{i=1}^n \eta_i\cdot\inner{\newunit_n}{\unit_i}
        =\frac{(n-1)\eta_1-\sum_{i=2}^n\eta_i}{\sqrt{n(n-1)}}\,.
    \end{equation*}
    Since, for any $v\in\RR$, $f_{\err{\sumpostdata}_1}(v)=f_{\eta'_n}\left(v\cdot\sqrt{\frac{n}{n-1}}\right)$, 
    \begin{equation*}
        \err{\sumpostdata}_1\sim\frac{(n-1)\eta_1-\sum_{i=2}^n\eta_i}{n}\,.
    \end{equation*}
    By independence of $\{\eta_i\}_{i\in[n]}$, it follows that 
    \begin{align*}
        \var{\err{\sumpostdata}_1}= 2\lambda^2\left(1-\frac{1}{n}\right).
    \end{align*}
\vspace{-4pt}
\end{proof} 

\begin{figure}[!t]
\centering
\includegraphics[width=0.6\linewidth]{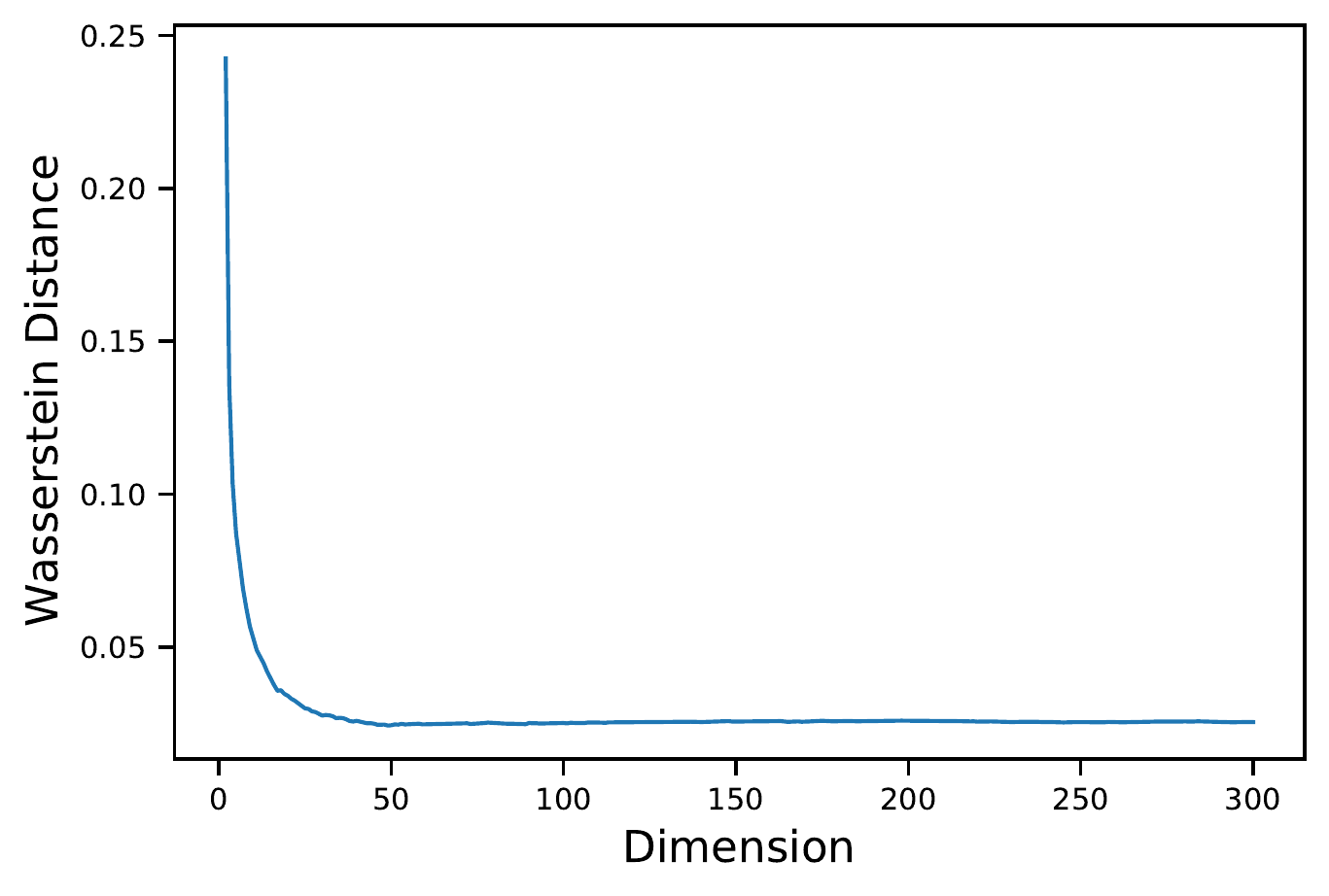}
\caption{Wasserstein Distance between the Laplacian Distribution and the Marginal Error Distribution.}
\label{fig:Wasserstein}
\vspace{-0.5cm}
\end{figure}

\begin{figure*}[!th]
\centering
\includegraphics[width=0.85\linewidth]{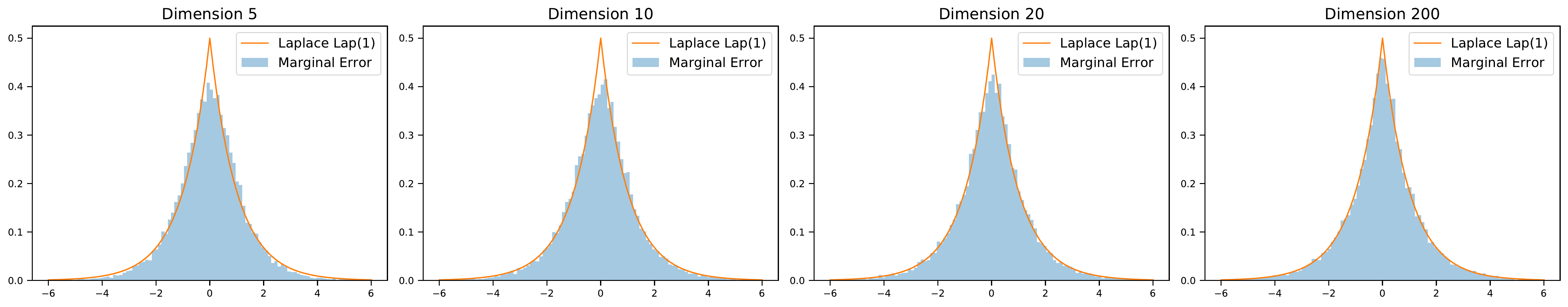}
\caption{Illustrating the Convergence Results of Theorem \ref{thm:conv}.}
\label{fig:convergence}
\end{figure*}

Figure \ref{fig:Wasserstein} highlights Theorem \ref{lem:err_dist}. It shows how the Wasserstein distance between the distributions of the Laplace residuals and the post-processed residuals. As the figure indicates, the Wasserstein distance decreases quickly as the problem dimension increases. 

Theorem \ref{lem:err_dist} also reveals some fundamental insights about post-processing. First, it shows that post-processing reduces the variance of the noise, while preserving differential privacy. In other words, post-processing in this setting does not introduce bias and leverages the public information (i.e., $b$) to reduce the variance. This is again a positive result as reducing the variance may reduce fairness issues when using the data in decision-making processes. Second, it shows that different aggregation sizes (i.e., different values of $n$) may lead to disparate impacts and fairness issues. Indeed, consider two counties $a$ and $b$ with approximately the same sizes which are aggregated differently: $a$ is aggregated with $n_a$ other counties, $b$ is aggregated with $n_b$, with $n_a \gg n_b$, and the aggregated data is public. Then the variance of the post-processed value for $a$ will be substantially larger than the the variance of the post-processed value of $b$, potentially creating situations where counties $a$ and $b$ will be treated fundamentally differently in decision-making processes.
Hence, although post-processing reduces variance, its application should take into account fairness issues. Once again, the key to ensure fairness is to make sure that quantities being released are of the same order of magnitude.

On the census data-release problem, when comparing the solutions returned by program $P_S$ for the states of Arizona---which has 15 counties---and Texas---which has 254 counties---it is found that both the theoretical and empirical difference in their variance to be roughly 6\%. This result highlights the importance of the finding.

The following results show that the marginal error converges in distribution to the Laplace distribution. So, even in the worst case, post-processing does not introduce 

\begin{theorem}\label{thm:lim_var}
    The variance of the resulting marginal error of program \eqref{sumprogram} is increasing in the dimension $n$ and
    converges to that of the marginal Laplacian noise added to the true data $\truedata$, as the dimension $n$ tends to infinity, i.e.,
    \begin{equation*}
        \lim_{n\to\infty}\var{\err{\sumpostdata}}=\var{{\rm Lap}(\lambda)}=2\lambda^2\,.
    \end{equation*}
\end{theorem}

\begin{theorem}
\label{thm:conv}
    As the dimension $n$ goes to infinity, the marginal error of  program \eqref{sumprogram} converges in distribution to 
    the marginal Laplacian noise , i.e.,
    \begin{equation*}
        \err{\sumpostdata}~\xrightarrow{~~d~~}~{\rm Lap}(\lambda)\,,\qquad {\rm as~~} n\to\infty\,.
    \end{equation*}
\end{theorem}
\begin{proof}
    By Lemma \ref{lem:err_dist}, the marginal error $\err{\sumpostdata}$ follows
    \begin{equation*}
        \frac{(n-1)\eta_1-\sum_{i=2}^n\eta_i}{n}\,,
    \end{equation*}
    where $\{\eta_i\}_{i\in[n]}$ are the $n$  i.i.d. random variables drawn from a Laplace distribution $\text{Lap}(\lambda)$.
    Let $\eta$ be a Laplacian random variable $\text{Lap}(\lambda)$: Its cumulative distribution function is
    \begin{equation*}
        \pr{\eta\leq v}=\begin{cases}
            \frac{1}{2}\exp\left(\frac{v}{\lambda}\right), & v\leq 0,\\
            1-\frac{1}{2}\exp\left(\frac{-v}{\lambda}\right), & v>0.
        \end{cases}
    \end{equation*}
    The cumulative distribution function of $(n-1)\eta_1/n$ is 
    \begin{equation*}
        \pr{\frac{n-1}{n}\eta_1\leq v}=\begin{cases}
            \frac{1}{2}\exp\left(\frac{nv}{(n-1)\lambda}\right), & v\leq 0,\\
            1-\frac{1}{2}\exp\left(\frac{-nv}{(n-1)\lambda}\right), & v>0.
        \end{cases}
    \end{equation*}
    Note that, for any $v\in\RR$,
    \begin{equation*}
        \lim_{n\to\infty}\pr{\frac{n-1}{n}\eta_1\leq v}=\pr{\eta\leq v}\,,
    \end{equation*}
    which implies that the random variable $(n-1)\eta_1/n$ converges to $\eta$ in distribution. By the Weak Law of Large Numbers, the sample mean among $\{\eta_i\}_{i\in\{2,\dots, n\}}$ converges in probability to 
    their common expectation $0$, i.e., for $\xi>0$,
    \begin{equation*}
        \lim_{n\to\infty}\pr{\left\vert\frac{\sum_{i=2}^n \eta_i}{n-1}\right\vert\geq \xi} = 0\,.
    \end{equation*}
    Additionally, for any $\xi>0$ and $n\geq 2$,
    \begin{equation*}
        \left\vert\frac{\sum_{i=2}^n \eta_i}{n}\right\vert\geq \xi\implies\left\vert\frac{\sum_{i=2}^n \eta_i}{n-1}\right\vert\geq \xi\,.
    \end{equation*}
    which implies that
    \begin{equation*}
        \pr{\left\vert\frac{\sum_{i=2}^n \eta_i}{n}\right\vert\geq \xi}\leq \pr{\left\vert\frac{\sum_{i=2}^n \eta_i}{n-1}\right\vert\geq \xi}\,.
    \end{equation*}
    By the squeeze theorem, the random variable $\sum_{i=2}^n \eta_i/n$ converges to 0 in probability, as the dimension
    $n$ goes to infinity. Since 
\begin{equation*}
    \frac{n-1}{n}\eta_1~\xrightarrow{~~d~~}~\eta\,,\qquad \frac{\sum_{i=2}^n \eta_i}{n}~\xrightarrow{~~p~~}~0,
\end{equation*}
by Slutsky's Theorem \cite{billingsley2013convergence}, it follows that 
\begin{equation*}
    \frac{(n-1)\eta_1-\sum_{i=2}^n\eta_i}{n}~\xrightarrow{~~d~~}~\eta\,\qquad {\rm as~~} n\to\infty.
\end{equation*}
\end{proof}

\noindent
Figure \ref{fig:convergence} illustrates Theorem \ref{thm:conv}. It depicts the convergence to the Laplace distribution as the dimension increases. It also shows how the variance decreases. Finally, it is also interesting to report some experimental results on census data and, in particular, the states of Arizona (population of 2,371,715 and 15 counties) and Texas (population of 8,887,839 and 254 counties). For $\lambda = 10$, the distribution variances are 186.67 and 199.21 for Arizona and Texas respectively. Over 80,000 experiments, the empirical variances were 186.88 and 199.32 respectively. These results clearly highly the influence of the problem dimension (i.e., the number of counties) on the variance. 

\section{Conclusion}

This paper was motivated by the recognition that the effect of post-processing
on the noise distribution is poorly understood: It took a first step towards understanding the theoretical and empirical properties of post-processing. Motivated by census applications, it studied the behavior of  \emph{projections} for domains where the feasibility space is specified by linear equations. The paper showed that, when non-negativity constraints are absent, the projection does not introduce bias. With non-negativity constraints, the paper presented an upper bound on the bias, providing insights on the type of problems for which the bias will be significant. 
The paper also provided a detailed analysis of the important sub-problem with one linear equation arising in hierarchical data-release problems. It fully characterized the residual distribution of the post-processed noise, showing that it converges towards the selected noise distribution when the dimension of the feasible space increases. This last result shed an interesting light on the effect of post-processing on the variance of the post-processed data. Indeed, in this case, post-processing reduces the variance
by exploiting the public information available. These results may have strong implications with respect to group fairness and should inform statistical agencies about the trade-off between the granularity of the released data, the bias, and the variance.





\bibliographystyle{plain}
\bibliography{aaai21.bib}

\end{document}